%% file: main.tex
\documentclass{article}

\usepackage{microtype}
\usepackage{graphicx}
\usepackage{subfigure}
\usepackage{booktabs} 

\usepackage{hyperref}
\usepackage{url}
\usepackage{enumitem}

\usepackage[accepted]{icml2024}

\usepackage{amsmath}
\usepackage{amssymb}
\usepackage{mathtools}
\usepackage{amsthm}
\usepackage{ulem}

\usepackage{multirow}
\usepackage{makecell}
\usepackage{xspace}

\usepackage[capitalize,noabbrev]{cleveref}
\crefname{equation}{Eq.}{Eqs.}
\crefname{theorem}{Theorem}{Theorems}
\Crefname{theorem}{Theorem}{Theorems}

\theoremstyle{plain}
\newtheorem{theorem}{Theorem}[section]

\theoremstyle{definition}

\theoremstyle{remark}

\usepackage[textsize=tiny]{todonotes}

\icmltitlerunning{Is DPO Superior to PPO for LLM Alignment? A Comprehensive Study}

\begin{document}

\twocolumn[
\icmltitle{Is DPO Superior to PPO for LLM Alignment? A Comprehensive Study}



\icmlsetsymbol{equal}{*}

\begin{icmlauthorlist}
\icmlauthor{Shusheng Xu}{Tsinghua}
\icmlauthor{Wei Fu}{Tsinghua}
\icmlauthor{Jiaxuan Gao}{Tsinghua}
\icmlauthor{Wenjie Ye}{openpsi}
\icmlauthor{Weilin Liu}{openpsi}\\
\icmlauthor{Zhiyu Mei}{Tsinghua}
\icmlauthor{Guangju Wang}{openpsi}
\icmlauthor{Chao Yu}{equal,Tsinghua}
\icmlauthor{Yi Wu}{equal,Tsinghua,openpsi,sqz}
\end{icmlauthorlist}

\icmlaffiliation{Tsinghua}{Tsinghua University, Beijing China}
\icmlaffiliation{openpsi}{OpenPsi Inc.}
\icmlaffiliation{sqz}{Shanghai Qi Zhi Institute, Shanghai, China}

\icmlcorrespondingauthor{Shusheng Xu}{xssstory@gmail.com}
\icmlcorrespondingauthor{Chao Yu}{zoeyuchao@gmail.com}
\icmlcorrespondingauthor{Yi Wu}{jxwuyi@gmail.com}

\icmlkeywords{Machine Learning, ICML}

\vskip 0.3in
]



\printAffiliationsAndNotice{\textsuperscript{*}Co-corresponding authors.}

\renewcommand{\emph}[1]{\textit{#1}}

\input{00_abstract}
\input{10_intro}
\input{20_related}

\input{30_prelim}
\input{40_analysis}

\input{50_experiment}
\input{60_conclusion}

\newpage
\section*{Impact Statements}

 Our study investigates a critical challenge in aligning Large Language Models (LLMs) with human preferences, emphasizing its societal impact, including the elimination of bias and the reduction of unfairness. The use of a public dataset ensures transparency, mitigating concerns related to privacy and ethical considerations. This research emphasizes our dedication to responsible AI practices, aiming to improve societal well-being by aligning LLMs with human values while upholding rigorous standards for privacy and ethics.



\nocite{langley00}

\bibliography{example_paper}
\bibliographystyle{icml2024}
\input{70_appendix}






\end{document}

%% file: 00_abstract.tex
\begin{abstract}
Reinforcement Learning from Human Feedback (RLHF) is currently the most widely used method to align large language models (LLMs) with human preferences. 
Existing RLHF methods can be roughly categorized as either \emph{reward-based} or \emph{reward-free}. 
Novel applications such as ChatGPT and Claude leverage \emph{reward-based} methods that first learn a reward model and apply actor-critic algorithms, such as Proximal Policy Optimization (PPO). 
However, in academic benchmarks, the state-of-the-art results are often achieved via \emph{reward-free} methods, such as Direct Preference Optimization (DPO).
\emph{Is DPO truly superior to PPO?} 
\emph{
Why does PPO perform poorly on these benchmarks?
}
In this paper, we first conduct both theoretical and empirical studies on the algorithmic properties of DPO and show that DPO may have fundamental limitations. 
Moreover, we also comprehensively examine PPO and reveal the key factors for the best performances of PPO in fine-tuning LLMs. Finally, we benchmark DPO and PPO across a collection of RLHF testbeds, ranging from dialogue to code generation. Experiment results demonstrate that PPO is able to surpass other alignment methods in all cases and achieve state-of-the-art results in challenging code competitions. Our code is publicly available at \url{https://github.com/openpsi-project/ReaLHF}.
\end{abstract}

%% file: 10_intro.tex
\section{Introduction}
\label{intro}


Large Language Models (LLMs) derive their extensive language patterns and knowledge through pre-training on substantial textual datasets~\cite{gpt3,gpt4,llama2,palm,palm2}. To leverage the formidable capabilities of LLMs in practical applications, a growing amount of research has underscored the importance of aligning these models with human preferences~\cite{llm-know-halluci,llm-know-llm,distract-llm,llm-bias1,llm-bias2}. Various methods have been developed for fine-tuning LLMs, with popular approaches including Supervised Fine-Tuning (SFT)~\cite{instrft-gp4} and Reinforcement Learning from Human Feedback (RLHF)~\cite{llm-rlhf-openai-1,llm-rlhf-openai-2,instrgpt}. Typically, fine-tuning involves two phases: SFT to establish a base model, followed by RLHF for enhanced performance. SFT involves imitating high-quality demonstration data, while RLHF refines LLMs through preference feedback.

Within RLHF, two prominent approaches are \emph{reward-based} and \emph{reward-free} methods. Reward-based methods, pioneered by OpenAI~\cite{instrgpt,llm-rlhf-openai-1,llm-rlhf-openai-2}, construct a reward model using preference data and then employ actor-critic algorithms like Proximal Policy Optimization (PPO) to optimize the reward signal. In contrast, reward-free methods, including Direct Preference Optimization (DPO)~\cite{dpo}, RRHF~\cite{rrhf}, and PRO~\cite{pro}, eliminate the explicit use of a reward function. DPO, a representative reward-free method, expresses the reward function in a logarithmic form of the policy and focuses solely on policy optimization.

Notably, the most successful applications like ChatGPT~\cite{chatgpt} and Claude~\cite{claude} are produced by the reward-based RLHF method PPO, while strong performances in academic benchmarks often result from the reward-free RLHF method DPO~\cite{dpo, mistral7b}. This discrepancy raises two fundamental questions: 1) \emph{Is DPO truly superior to PPO in the RLHF domain?} and 2) \emph{Can the performance of PPO be substantially improved in common RLHF benchmarks?} In this paper, we delve into these questions. Through theoretical and empirical analysis, we uncover the fundamental limitations of DPO and explore critical factors that enhance the practical performance of PPO in RLHF. 

First, our theoretical examination reveals that 
DPO might find biased solutions that exploit out-of-distribution responses.
Empirically, we demonstrate that the performance of DPO is significantly affected by the distribution shift between the model outputs and the preference dataset.
Second, we perform ablation studies on the algorithmic components of PPO and discover a collection of critical factors for PPO's best RLHF performances, including advantage normalization, large batch size, and exponential moving average update for the reference model. Finally, we validate our findings through extensive experiments, including dialogue generation tasks and more challenging code generation tasks. These experiments feature diverse feedback types and difficulty levels.
The results indicate that PPO consistently outperforms DPO across all experiments.
Particularly, in the most challenging code competition tasks, PPO achieves state-of-the-art results. Specifically, on the CodeContest dataset~\cite{li2022competition}, our PPO model with 34B parameters outperforms AlphaCode-41B~\cite{li2022competition}, exhibiting a 10@1k improvement from 16.4\% to 22.4\%.

%% file: 20_related.tex
\section{Related Work}
\label{related}

Large language models (LLMs) trained on large datasets acquire surprising capabilities~\cite{gpt3,gpt4,llama2,palm,palm2,scalinglaw,gpt3}.
To leverage these capabilities to real applications, pre-trained LLM is further fine-tuned on specific tasks~\cite{gpt2,flan-t5,ul2}. Through fine-tuning with popular approaches such as SFT and RLHF, LLMs demonstrate impressive performance on established benchmarks~\cite{llama2,gpt4}, aligning further with human preferences and societal well-being~\cite{ai-modern-book,humancompat}.

This paper concentrates on RLHF methods, which can be broadly categorized into reward-based and reward-free approaches. 
Reward-based methods entail training a reward model on preference data in an initial phase~\cite{rwdscaling,llm-rlhf-openai-1,llm-rlhf-openai-2,instrgpt}. Subsequently, this learned reward model is utilized to provide a reward signal for online Reinforcement Learning (RL) algorithms such as PPO~\cite{ppo}. 
There exist previous works that have studied these methods through hyper-parameter tuning and analyzed the effects of the quality reward model quality~\cite{fudan-rlhf,challenge-rlhf}. 
In contrast, reward-free methods offer a simpler training procedure by directly training LLMs on preference data or ranking data to distill human preference~\cite{rrhf,dm-reject-sampling,llama2,dpo,pro,RAFT,hong2024reference}. 
Among these reward-free methods, DPO~\cite{dpo} has demonstrated strong performances and become popular in the community~\cite{mistral7b,chen2024self,selfrwd-llm}. Recent work discussed the performance gap of DPO and PPO on synthetic contextual bandits~\cite{li2023policy}.
In this paper, We analyze the limitations of DPO theoretically and empirically, and explore the key factors for PPO training.

Concurrent efforts have been undertaken to avoid reward model overoptimization~\cite{rw-warm}, facilitate alignment data generation~\cite{rlaif,rlcd}, and implement resource-efficient RLHF systems~\cite{dschat,lora-ppo}.
These works complement our study and can be seamlessly integrated into our implementation.
Previous works have explored the implementation details of PPO for LLMs~\cite{fudan-rlhf,nlpo}. Our paper extends its investigations with additional RLHF techniques, optimizing PPO performance to surpass its reward-free counterpart, DPO. 
Our work is also closely related to studies on algorithm implementation in the RL community~\cite{ppo-impl,whatmattersonpolicy,mappo}. However, our findings provide further insights into fine-tuning LLMs with a model size of up to 34B parameters.

%% file: 30_prelim.tex
\section{Preliminary}
\label{prelim}

\newcommand{\x}{\mathbf{x}}
\newcommand{\y}{\mathbf{y}}
\newcommand{\E}[2]{\mathbb{E}_{#1}\left[{#2}\right]}
\newcommand{\piref}{\pi_\mathrm{ref}}
\newcommand{\prefdata}{\mathcal D}

\textbf{Language Model.} We consider an LLM as a policy $\pi_\theta(\y\mid\x)$ parameterized by $\theta$. $\pi_\theta$ is designed to follow user instructions $\x\in\mathcal{X}$  to generate a text response $\y\in\mathcal{Y}$. We only consider single-round conversations to simplify notations. Given a prompt $\x$, the LLM $\pi_\theta$ will generate response $\y$ in an auto-regressive manner:
\begin{equation}
\pi_\theta\left(\y\mid \x\right)=\prod_t \pi_\theta \left(y_t\mid \x, \y_{<t}\right),
\end{equation}
where $y_t$ is the $t$-th token in the response and $\y_{<t}$ is tokens in the response before $y_t$.

\textbf{SFT.} As an initial phase of alignment, the pre-trained model is enforced to imitate high-quality demonstration data (dialogue, summarization, etc.), which is usually referred to as Supervised Fine-Tuning (SFT).

\textbf{RLHF.} To further align the SFT model $\pi_\theta$ with human preference, prior works~\cite{llm-rlhf-openai-1,instrgpt} proposed the Reinforcement Learning from Human Feedback (RLHF) procedure, which maximizes the following objective,
\begin{equation}
    J_r(\pi_\theta)=\E{\x\sim p_\mathrm{data},\y\sim\pi_\theta}{r(\x,\y)-\beta\log\frac{\pi_\theta(\y\mid\x)}{\pi_\mathrm{ref}(\y\mid\x)}}.\label{eq:obj}
\end{equation}
where $r$ is the reward function reflecting human preferences. $r$ takes a prompt and the corresponding response as input and outputs a scalar value. $\pi_\mathrm{ref}$ is the reference model used for regularizing $\pi_\theta$ with Kullback–Leibler divergence. $\beta$ is a constant to control the degree of regularization.

In the rest of this section, we will introduce two representative algorithms to optimize \cref{eq:obj}: a reward-based approach, PPO, and a reward-free approach, DPO.

\textbf{PPO.} We can directly adopt standard reinforcement learning methods for \cref{eq:obj}. In this paper, we chose PPO as the training algorithm.
 When $r$ is unknown, a reward model $r_\phi\in R$ is first learned from human-labeled data to approximate $r$.
A common practice is to collect a dataset of preference pairs $\prefdata =\{(\x,\y_w,\y_l)\}$. $\y_w$ and $\y_l$ are responses to $\x$ and marked as ``win'' and ``lose'' by human respectively.
The distribution of the preference dataset is assumed to follow the Bradley-Terry model~\cite{bradley1952rank,rlpref-nips2017}, i.e., the probability of response $\y_w$ is better than $\y_l$ is given by
\begin{align}
    \mathbb{P}_\phi(\y_w\succ\y_l\mid\x)&=\frac{\exp\left(r_\phi(\x,\y_w)\right)}{
    \exp\left(r_\phi(\x,\y_w)\right) + \exp\left(r_\phi(\x,\y_l)\right)
    }\notag\\
    &=\sigma\left(r_\phi(\x,\y_w)-r_\phi(\x,\y_l)\right).
    \label{eq:bradley-terry}
\end{align}
where $\sigma$ is the sigmoid function.
Given $\prefdata$, $r_\phi$ is trained by minimizing the negative log-likelihood of \cref{eq:bradley-terry}:
\begin{align}
    \mathcal L_R(r_\phi)=-\E{(\x,\y_w,\y_l)\sim \prefdata}{\log \sigma(r_\phi(\x,\y_w)-r_\phi(\x,\y_l))}\label{eq:rew-loss}
\end{align}
After a reward model $r_\phi$ is obtained, $r$ is replaced with $r_\phi$ and $J_{r_\phi}(\pi_\theta)$ could be explicitly optimized with online RL algorithms.
We note that there exist cases when a ground-truth reward is available, and thus reward modeling becomes unnecessary~\cite{bertscore,BLEURT,nlpo}. In these cases, the reward function can be directly incorporated into \cref{eq:obj}.
While we acknowledge other actor-critic algorithms can also be feasible~\cite{a3c,sac}, we follow the mainstream work~\cite{llm-rlhf-openai-1,llm-rlhf-openai-2} and focus on PPO~\cite{ppo} for our analysis in this paper.

\textbf{DPO.}
Instead of learning a reward model,
Direct Preference Optimization (DPO)~\cite{dpo} optimizes the policy $\pi_\theta$ over preference data. DPO derived the closed-form solution of \cref{eq:obj}, which reveals the relationship between the reward $r(\x,\y)$ and the optimal language model $\pi^*(\y\mid\x)$:
\begin{equation}
    \pi^*(\y\mid\x)=\frac{1}{Z(\x)}\piref(\y\mid\x)\exp\left(\frac{1}{\beta}r(\x,\y)\right),\label{eq:closeform-sol}
\end{equation}
where $Z(\x)$ is a partition function that only depends on prompt $\x$.
According to \cref{eq:closeform-sol}, if $\pi_\theta$ maximizes $J_{r_\phi}(\pi_\theta)$, the underlying reward can be derived with
\begin{equation}
    r_\phi(\x,\y)=\beta\log \frac{\pi_\theta(\y\mid\x)}{\piref(\y\mid\x)}+C(\x).\label{eq:dpo-reward}
\end{equation}
where $C:\mathcal{X}\rightarrow \mathbb R$ is a scalar function. 
This enables us to
reparameterize \cref{eq:rew-loss} with the policy $\pi_\theta$, and then we can drive the DPO loss that directly optimizes $\pi_\theta$, i.e.,
\begin{footnotesize}
\begin{align}\label{eq:dpo-loss}
    \mathcal L_\mathrm{DPO}&(\pi_\theta)=-\mathbb{E}_{(\x,\y_w,\y_l)\sim\mathcal D}\\
    &
    \left[\log\sigma\left(\beta\left(
    \log\frac{\pi_\theta(\y_w\mid\x)}{\piref(\y_w\mid\x)}-
    \log\frac{\pi_\theta(\y_l\mid\x)}{\piref(\y_l\mid\x)}
    \right)\right)\right]\notag.
\end{align}
\end{footnotesize}
We remark that although \citet{dpo} performs a single-round DPO over the preference dataset, some recent works also adapt DPO to an iterative variant with a learned reward model~\cite{xiong2023iterative,selfrwd-llm}.
We also investigate the performance of iterative DPO.

%% file: 40_analysis.tex
\section{Understanding the Limitation of DPO}
\label{analysis}
In this section, we demonstrate that DPO may not be superior to PPO. 
Firstly, we theoretically demonstrate issues with the DPO training objective. Secondly, we illustrate that DPO is more susceptible to out-of-distribution (OOD) data through a synthetic example. 
Lastly, through experiments on a real preference dataset, we validate that the performance of DPO can be improved by mitigating the distribution shift between the model outputs and the preference dataset.

\subsection{Theoretical Analysis}
\label{subsec:analysis-dpo}

It is well-known that PPO could exploit potential failures in the learned reward model to achieve high rewards without meeting the actual human preference, often manifested as erroneous~\cite{lewis2017deal} or overly complex outputs~\cite{singhal2023long}. 
We argue that, though DPO avoids reward modeling, DPO has a similar generalization issue. 
In the following theorem, we will show that any solution found by PPO also minimizes the DPO objective \cref{eq:dpo-loss}, and thus, any solution found by PPO that exploits the reward model can also be found by DPO. 
Furthermore, DPO might discover solutions exploiting out-of-distribution data, posing a risk of deviating excessively from the reference policy even when the reference policy aligns well with human preferences.

\begin{table}[]
    \centering
    \begin{tabular}{c|ccc}
    \toprule
         Action & $\y_1$ & $\y_2$ & $\y_3$\\
         \midrule
         $\piref$ & $0.5$ & $0.5$ & $0$ \\
         $D_\mathrm{pref}$ & \multicolumn{3}{c}{$\{(\y_w=\y_1,\y_l=\y_2)\}$} \\
         \midrule
         $\pi_\mathrm{DPO}$ & $0.1$ & $0.0$ & $0.9$ \\
         \midrule
         $\pi_\mathrm{PPO}$ & 1 & 0 & 0 \\
         \bottomrule
    \end{tabular}
    \caption{A state-less counter-example with three actions when DPO can minimize the loss but produce an unexpected policy. PPO will not produce $\pi_\mathrm{DPO}$ because $\piref$ enforces the probability of outputting $\y_3$ is zero.}
    \label{tab:dpo-counterexample}
\end{table}

\begin{theorem}\label{thm:main}
Given a ground-truth reward $r$ and a preference dataset $\mathcal D$, let $\Pi_\mathrm{PPO}$ be the class of policies induced by training reward model $r_\phi$ over $\mathcal D$ and running PPO to optimize $J_{r_\phi}(\theta)$. Let $\Pi_\mathrm{DPO}$ be the class of policies induced by minimizing DPO objective \cref{eq:dpo-loss}. 
We have the following conclusion: $\Pi_\mathrm{PPO}$ is a proper subset of $\Pi_\mathrm{DPO}$.
\end{theorem}

\begin{proof}
We first prove that $\Pi_\mathrm{PPO}$ is a subset of $\Pi_\mathrm{DPO}$, i.e. $\Pi_\mathrm{PPO}\subseteq \Pi_\mathrm{DPO}$.
Let $R$ be the class of reward models that minimizes reward learning loss \cref{eq:rew-loss}. We note there is an one-to-many mapping between $\Pi_\mathrm{PPO}$ and $R$
according to \cref{eq:dpo-reward}. Without loss of generality, we omit the scalar factor $C(\x)$ and define $f$ to be an one-to-one mapping from a policy to a reward function $f(\pi)(\x,\y)=\beta\log\frac{\pi(\y|\x)}{\pi_{ref}(\y|\x)}$.

We will show
that the minimum DPO loss is the same as the minimum reward learning loss, i.e., $\min_r \mathcal L_{R}(r)=\min_\pi \mathcal L_\mathrm{DPO}(\pi)$. We can show that $\min_\pi \mathcal L_\mathrm{DPO}(\pi)=\min_\pi \mathcal L_{R}(f(\pi))\ge \min_r \mathcal L_{R}(r)$. And for a minimizer $r^*$ of $\mathcal L_R(r)$, we can construct a policy $\pi^*$ from $r^*$ by \cref{eq:closeform-sol} and then $\min_\pi \mathcal L_\mathrm{DPO}(\pi)\le \mathcal L_\mathrm{DPO}(\pi^*)=\min_r \mathcal L_{R}(r)$. Therefore the reward learning loss achieved by reward models in $R$ and DPO loss achieved by policies in $\Pi_\mathrm{DPO}$ are the same, i.e. $\forall r_\phi\in R, \pi_\mathrm{DPO} \in \Pi_\mathrm{DPO}, \min_{r}L_{R}(r)=\mathcal L_{R}(r_\phi)=\mathcal L_\mathrm{DPO}(\pi_\mathrm{DPO})=\min_{\pi}L_\mathrm{DPO}(\pi)$. 

For any solution found by PPO, $\pi_\mathrm{PPO}\in \Pi_\mathrm{PPO}$, the reward $r^*=f(\pi_\mathrm{PPO})$ satisfies that $\pi_\mathrm{PPO}$ is a maximizer of $J_{r^*}(\pi)$ and $\pi_\mathrm{PPO}$ can be represented by $r^*$ with 
\begin{equation}
    \pi_{\mathrm{PPO}}(\y\mid\x)=\frac{1}{Z(\x)}\piref(\y\mid\x)\exp\left(\frac{1}{\beta}r^*(\x,\y)\right).\label{eq:inproof-ppo-closeform-sol}
\end{equation}
Substituting $\pi_\mathrm{PPO}$ with \cref{eq:inproof-ppo-closeform-sol} in $\mathcal L_\mathrm{DPO}(\pi_\mathrm{PPO})$, we get $\mathcal L_\mathrm{DPO}(\pi_\mathrm{PPO})=\mathcal L_R(r_\phi)$. Therefore, $\pi_\mathrm{PPO}$ also minimizes the DPO loss, which implies $\pi_\mathrm{PPO}\in \Pi_\mathrm{DPO}$.

Next, we show that $\Pi_\mathrm{PPO}$ is a proper subset of $\Pi_\mathrm{DPO}$, i.e. $\Pi_\mathrm{PPO}\subsetneq \Pi_\mathrm{DPO}$ with a counter-example as shown in \cref{tab:dpo-counterexample}. In this counter-example, we will show that there exists a solution found by DPO, $\pi_\mathrm{DPO}\in\Pi_\mathrm{DPO}$, that does not maximize the RL objective of PPO \cref{eq:obj}.
Consider a simple state-less case with three actions, but the preference dataset only contains a single pair comparison between $\y_1$ and $\y_2$. Denote the probability of DPO policy $\pi_\mathrm{DPO}$ outputting the first two actions as $a$ and $b$. The DPO loss in this scenario is given by $\mathcal{L}_\mathrm{DPO}=\log(1+(\frac{b}{a})^\beta)$, which can be minimized as long as $b=0$. A possible optimal policy produced by DPO is shown in the third row of \cref{tab:dpo-counterexample}, which has a $0.1$ probability to output $\y_1$ and a $0.9$ probability to output $\y_3$. This policy cannot be produced by PPO because $\piref$ enforces $\pi_\mathrm{PPO}$ to assign $0$ probability to $\y_3$ according to \cref{eq:inproof-ppo-closeform-sol}.

\end{proof}

We remark that the root cause of reward misspecification is the narrow distribution coverage of the preference dataset.
The learned reward model may assign a high value to out-of-distribution (OOD) samples and has the potential to be exploited during the RL process.
\textbf{Although DPO avoids training the reward model, it still suffers from the misspecification issue on OOD samples but in a different manner.}
Specifically, DPO can develop
a biased distribution favoring unseen responses, directly
impacting quality of the learned policy.
By contrast, PPO can leverage prompt-only data and generate responses beyond the preference dataset distribution.
During training, KL divergence between $\pi_\theta$ and $\piref$ can provide additional regularization for PPO on these generated samples. 

\begin{figure}
    \centering
    \includegraphics[width=\columnwidth]{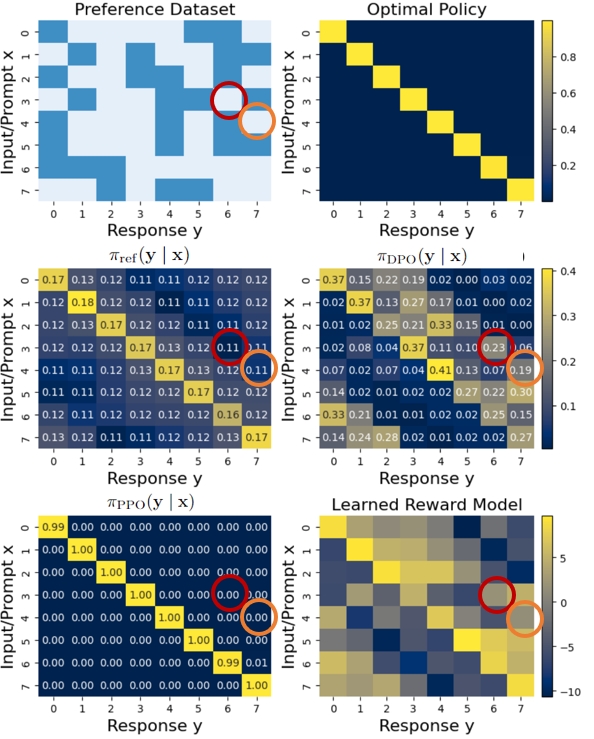}
    \vspace{-7mm}
    \caption{Preference dataset coverage, policy probability distributions of $\piref$, $\pi_\mathrm{PPO}$, $\pi_\mathrm{DPO}$, and the value of learned rewards in the synthetic scenario. In the first figure, dark color represents data present in preference data, while light means the data points are not included.
    Although data points marked with red circles and orange circles are not covered by the preference dataset, DPO assigns higher probabilities of these data points compared with the reference model. PPO assigns low probability to the marked data points and learns the optimal policy.}
    \label{fig:synthetic-heatmap}
    \vspace{-5mm}
\end{figure}

\subsection{Empirical Validation in A Synthetic Scenario}
We design a synthetic scenario to validate \cref{thm:main} in practice.
We create discrete spaces of prompts and responses, both of size 8. The policy $\pi_\theta$ and reward model $r_\phi$ are modeled as MLPs, which take a one-hot vector as input and output a categorical distribution of overall responses.
We manually enforce the optimal response to be diagonal indices.
The preference dataset is randomly created under this constraint and only covers limited preference pairs for each input. 
The resulting policies of DPO and PPO are shown in \cref{fig:synthetic-heatmap}.
We can see that in practice, DPO and the learned reward model can assign high values to the response out of the distribution of preference dataset, which are marked using circles.
In the case of DPO, the final model may assign higher probabilities than the reference model to these responses, which is not desirable as performance improvement on OOD responses could not be guaranteed. For example, in the red circles, DPO increases the probability from 0.11 to 0.23.
In contrast, though the reward model has a similar misspecification issue, PPO can alleviate the issue with explicit KL regularization w.r.t. the reference model.

\textbf{Practical Remark:}
From the analysis in this section, we attempt to provide insights to understand the performance of DPO in practice --- 
\textbf{DPO is prone to generating a biased policy that favors out-of-distribution responses, leading to unpredictable} behaviors.
We will further validate these insights through an experimental study involving LLMs on real preference datasets.

\subsection{Experiments on Real Preference Datasets}
\label{subsec:analysis-dpo-exp}

\newcommand{\sftalpaca}{{\emph{SFT (Alpaca)}}\xspace}
\newcommand{\sftsafe}{{\emph{SFT (Safe)}}\xspace}

In this section, we conduct experiments on real preference datasets and
investigate two aspects that may influence DPO performance, including the base model and preference data used for DPO training.

\paragraph{Experimental Setup.} 
We perform our experimental analysis on \textbf{SafeRLHF} dataset~\cite{dai2023safe}. In this dataset, preference pairs have the form $(\x,\y_1,\y_2,l_\mathrm{h},l_\mathrm{s},b_1,b_2)$, where $l_\mathrm{h}$ and $l_\mathrm{s}$ are a preference labels over $\y_1$ and $\y_2$ in terms of helpfulness and safety, respectively, which could be either 1 or 2. $b_1$ and $b_2$ are binary safety labels of these two responses, which could be positive or negative.
With this dataset, our objective is to train an LLM that prioritizes safety over helpfulness in content generation. Specifically, in constructing the preference dataset, our preference is for the more helpful response when both options are considered safe (i.e., $l_h$ if $b_1$ and $b_2$ are both positive). Otherwise, our preference shifts towards the safer one (i.e., $l_s$).
Following \citet{dai2023safe}, the base model is trained on the Alpaca~\cite{taori2023stanford} open-source dataset with SFT, denoted as \sftalpaca. We use the evaluation models released by the official codebase to evaluate the helpfulness and harmfulness. We remark that these official evaluation models are not involved during the training.
Our experimental study is shown in \cref{tab:ood}.

\input{tables/dpo-study}
\input{tables/ppo_ablation}

\textbf{Impact of The Base Model.} 
When using \sftalpaca as the base and reference model, we find that DPO performs poorly, producing only a $55.4\%$ safety rate and low helpfulness reward. We hypothesize that this is caused by the distribution shift between the training data of the base model, i.e., the Alpaca dataset, and the preference data, i.e., the SafeRLHF dataset.
To study the impact, we further fine-tune \sftalpaca on the SafeRLHF dataset with safe responses to obtain \sftsafe. We then use \sftsafe as the reference model to re-train DPO from scratch. As shown in \cref{tab:ood}, resolving the distribution shift issue essentially increases the safety rate by $16.4\%$ and the helpfulness reward from $-4.19$ to $-1.62$. 

\textbf{Sensitivity to Preference Data.}
There exist pairs $(\x,\y_1,\y_2)$ in the SafeRLHF dataset where both $\y_1$ and $\y_2$ have the same safety label.
After filtering out the dual-unsafe and dual-safe preference data in the dataset, the trained model could obtain a much higher safety rate. However, filtering the dual-safe preference data would largely hurt the performance of helpfulness. These results suggest that while DPO may derive advantages from eliminating noise or controversies in the training data, excessively discarding high-quality data could be detrimental to DPO performance. 

\newcommand{\dpoiter}{\emph{DPO-Iter}\xspace}

\textbf{Impact of Preference Data Distribution.} While mitigating the distribution shift can be done with additional SFT, we also investigate whether collecting additional data with the base model could bring benefit. 
Specifically, instead of using the existing preference data, we generate new responses with \sftsafe and use a learned reward model for preference labeling. We further repeat this process and iteratively set the reference model as the latest DPO model in the last iteration. We denote this method as \dpoiter.
Remarkably, \dpoiter achieves a comparable safety rate with PPO. This experiment again demonstrates that DPO could be improved by mitigating the distribution shift. However, it also obtains a much lower helpfulness reward compared to PPO.

\input{tables/apps_bsz}

\textbf{Practical Remark:}
The performance of DPO could be improved by mitigating the distribution shift between the model and the preference dataset. To alleviate the issue of distribution shift and noisy data, 
we suggest adopting the iterative DPO method. One should \textbf{carefully} annotate the model-generated samples each time and then proceed to the next round of training. However, we will demonstrate in Sec.~\ref{experiment} that even with a nearly perfect annotator, the performance of DPO remains unsatisfactory in challenging tasks such as code generation.


\section{Key Factors to PPO for RLHF}
\label{analysis-ppo}

In this section, we investigate the key factors to the RLHF performance of PPO. We find three key techniques: (1) advantage normalization~\cite{stable-baselines3}, (2) large-batch-size training~\cite{mappo}, and (3) updating the parameters of the reference model with exponential moving average~\cite{instrgpt}. The first two techniques are widely adopted by the RL community but are not well-studied in the field of RLHF. The third is a technique that has received limited discussion in the literature, involving the gradual update of the reference model through an exponential moving average~\cite{instrgpt}. This particular approach has the potential to yield additional performance enhancements.

\input{tables/hh_main}

\paragraph{Implementation Details.} Our PPO implementation is based on DeepSpeed-Chat~\cite{dschat}, except that (1) we use a scalar reward for each response instead of dense rewards assigned on each token and (2) we omit the auxiliary SFT loss during PPO training because of the limited amount of data.
This implementation includes common PPO techniques such as value loss clip and generalized advantage estimation (GAE)~\cite{gae}. We list experiment details in Appendix~\ref{appendix:ppo_details}.

\textbf{Experimental Setup.} 
Our ablation experiments for PPO are carried out on a dialogue task \textbf{HH-RLHF}~\cite{bai2022training} as well as two code generation tasks: \textbf{APPS}~\cite{hendrycks2021measuring} and \textbf{CodeContest}~\cite{li2022competition}. 
HH-RLHF is a preference dataset in the form defined in \cref{prelim} that aims to train a helpful and harmless LLM. 
APPS and CodeContest datasets are competitive programming datasets.
Given a problem, the LLM should output a piece of executable code to solve this problem. The correctness is verified by test cases in the dataset, which can then generate reward signals or preference pairs for PPO and DPO training, respectively. We remark that these two types of tasks feature different types of reward signals: preference and direct reward feedback.
The complete experimental setup is listed in \cref{experiment}. 
In the experiment result, we denote advantage normalization as \emph{Adv. Norm.}, large batch-size training as \emph{LargeBatch} and exponential moving average of reference model update as \emph{Ref. EMA}. 

\textbf{Analysis.}
The result of the ablation study is shown in \cref{tab:ablation-all}. In \cref{tab:ablation-all}, with a small batch size, baseline PPO improves over the SFT model on HH-RLHF and CodeContest dataset but shows significant performance degradation on the APPS dataset. Advantage normalization stabilizes PPO training and improves the performance of PPO. The most significant benefit is brought by using a large batch size, especially on code generation tasks. 
Lastly, using the exponential moving average for the reference model also brings additional benefits.
The intuition behind this is that while the main LLM of PPO is rapidly changing, the reference model should also be updated accordingly. Otherwise, the learned model may be strongly regularized to be close to the SFT model, which can hurt performance in challenging tasks.
\cref{fig:apps-bsz} further demonstrates that increasing the batch size of PPO consistently improves the performance across all difficulty levels in the APPS dataset. 
We also highlight that utilizing a small batch size, such as 64, in PPO training could negatively impact the performance of the base SFT model, resulting in a 33.7\% performance level on the introductory scale.
We remark that our findings are consistent with those developed in the RL community~\cite{mappo}.

%% file: tables/dpo-study.tex
\renewcommand{\sftalpaca}{{\emph{SFT (Alpaca)}}\xspace}
\renewcommand{\sftsafe}{{\emph{SFT (Safe)}}\xspace}

\begin{table}[ht]
\centering
\small
\begin{tabular}{c|ccc}
\toprule
   & $\Delta$Help. $\uparrow$ & Harm. $\downarrow$  & S.R. $\uparrow$ \\ \midrule
SFT (Alpaca) & -2.62  & 1.50 & 41.6\%  \\
\midrule
PPO & 1.69  & -12.08 & 99.5\%      \\
+ SFT (Safe)& 4.47 &  -12.33 & 99.6\%\\
\midrule
DPO & -4.19 & -0.97  & 55.4\%      \\
+ SFT (Safe) & -1.62 & -3.50  & 71.8\%      \\
+ filter dual-unsafe & 2.46  & -4.88  & 80.8\%      \\ 
+ filter dual-safe & -2.86  & -6.82 & 95.8\%      \\ 
\midrule
 DPO Iter.1 & -3.22 & -5.23 & 86.7\%\\
 DPO Iter.2  & -3.27 & -8.83  & 99.7\% \\
 DPO Iter.3 & -3.26 & -10.21 & 99.9\%\\
 DPO Iter.4 & -2.96 & -11.07 & 99.9\%\\
\bottomrule
\end{tabular}
\caption{The impact of training data on DPO. We first train Llama-2-7B on the Alpaca open-source dataset and obtain \sftalpaca. Then the SFT model is trained with DPO and PPO. DPO performs poorly due to distribution mismatch and noises.
These issues can be resolved by (1) additional SFT on the preference dataset (\sftsafe), (2) filtering out controversy and noisy preference pairs, and (3) generating new responses and using a learned reward model to label the preference data for iterative DPO training.
}
\vspace{-5mm}
\label{tab:ood}
\end{table}

%% file: tables/ppo_ablation.tex
\begin{table*}[ht]
\centering
\begin{tabular}{c|c|ccc|ccc}
\toprule

Task & HH-RLHF & \multicolumn{3}{c|}{APPS} & \multicolumn{3}{c}{CodeContest} \\
\midrule
\multirow{2}{*}{Metric} & \multirow{2}{*}{\makecell{OpenAssaint\\ Reward}} & \multirow{2}{*}{\makecell{Intro.\\ pass@5}} & \multirow{2}{*}{\makecell{Inter.\\ pass@5}} & \multirow{2}{*}{\makecell{Comp.\\ pass@5}} &  \multirow{2}{*}{\makecell{pass@10}} & \multirow{2}{*}{\makecell{pass@100}} & \multirow{2}{*}{\makecell{pass@1k}} \\
& & & & & &\\
\midrule
SFT & 0.532  & 38.6\% &  10.1\% & 3.9\% & 0.9\% & 4.3\% & 12.0\% \\ 
\midrule
baseline PPO  & 0.706 & 18.0\% & 2.4\% & 1.1\% & 4.3\% & 6.0\% & 7.7\% \\
+ Adv.Norm.  & 0.716 & 38.1\%  & 11.4\% & 4.6\% & 6.8\% & 9.4\% & 15.4\% \\
+ Large.Batch.  & 0.716 & 42.3\% & 14.6\% & 7.5\% & 5.1\% & 12.8\% & 19.6\% \\
+ Ref.EMA & \textbf{0.718} & 44.4\% & 18.0\% & 9.1\% & 6.8\% & 13.7\% & 21.4\%  \\
\bottomrule
\end{tabular}
\label{tab:ablation-all}
\caption{Ablation study of PPO on different tasks. Baseline PPO is trained with a batch size of 64. Specifically, for the HH-RLHF task, the base model employed is Llama2-7B. In the case of APPS and CodeContest tasks, the base model utilized is CodeLlama-34B.}
\vspace{-5mm}
\end{table*}

%% file: tables/apps_bsz.tex
\begin{figure}[ht]
    \centering
    \includegraphics[width=\linewidth]{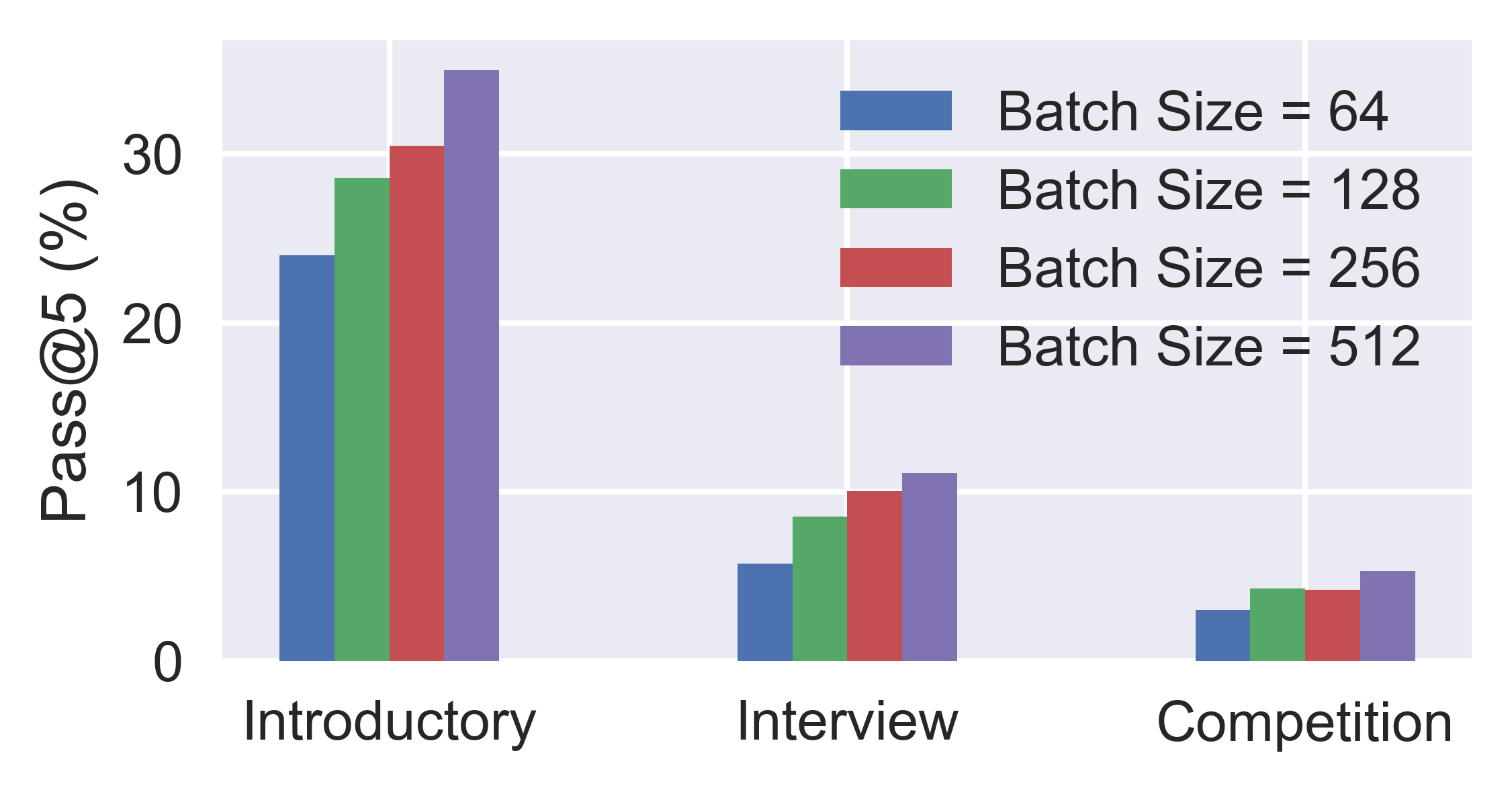}
    \vspace{-5mm}
    \caption{Performance of PPO on APPS dataset under different batch sizes. The base LLM is CodeLlma-13B. ``Introductory'', ``Interview'' and ``Competition'' represent three levels of difficulty.}
    \label{fig:apps-bsz}
\end{figure}

%% file: tables/hh_main.tex
\begin{table*}[h!]
\centering
\small
\begin{tabular}{c|c|ccc|ccc}
\toprule
       & OpenAssistant & \multicolumn{3}{c|}{ Tested V.S. Chosen} & \multicolumn{3}{c}{ Tested V.S. SFT} \\
       & \makecell{Reward} & Tested Win $\uparrow$ & Tie & Chosen Win $\downarrow$ &Tested Win $\uparrow$ & Tie & SFT Win $\downarrow$ \\
\midrule
RRHF & 0.523 &28  & 33 & 39 &  29 & 37 & 34 \\
PRO & 0.529 &37  & 26 & 37 & 34 & 33 & 33  \\
DPO & 0.611 &55  & 21 & 24 & 53 & 31 & 16  \\
DPO-Iter & 0.678 & 55 & 18 & 27 & 54 & 33 & 13 \\
PPO & \textbf{0.718} & 57 & 21 & 22 & 58 & 29 & 13 \\
\bottomrule
\end{tabular}
\caption{Results on the HH-RLHF test set. The evaluation metrics include the OpenAssistant rewards and the win rate of models against the chosen responses and SFT model outputs. The OpenAssistant reward model is not used during the training process. Note that DPO is trained on the preference data in the dataset, while Iter. DPO is trained on self-generated responses, using a reward model for labeling.}
\label{tab:hh-main}
\end{table*}

\begin{table}[ht]
\centering
\small
\begin{tabular}{c|ccc}
\toprule
                    & PPO Win & Tie & DPO Win \\ \midrule
PPO V.S. DPO   & 42      & 28  & 30       \\
PPO V.S. DPO-Iter   & 36      & 36  & 28       \\
\bottomrule
\end{tabular}
\caption{On HH-RLHF, we use GPT-4 to compare the outputs of the PPO and DPO models.}
\label{tab:hh-vs}
\end{table}

%% file: 50_experiment.tex
\section{Benchmark Results}
\label{experiment}

In this section, we conduct experimental validations to evaluate the performances of both DPO and PPO. Initially, our experiments focus on general dialogue tasks, specifically HH-RLHF and SafeRLHF. The primary goal is to improve the effectiveness of LLM by promoting constructive interactions and mitigating detrimental components within the model. Additionally, our investigation extends to demanding code generation tasks, namely APPS and CodeContest. 

\input{tables/safe_main}

\input{tables/apps_main}
\input{tables/ccs_main}


\textbf{HH-RLHF}~\cite{bai2022training} dataset consists of human preferences on AI assistant responses, encompassing 170k comparisons. In this dataset, we conduct experiments based on Llama2-7B.
We evaluate the trained models using the OpenAssistant reward model\footnote{https://huggingface.co/OpenAssistant/oasst-rm-2-pythia-6.9b-epoch-1}. Note that this model is only used for evaluation and is not involved during training. In addition, we adopt GPT-4 to compare the responses of different models. The prompt and evaluation details are listed in Appenidx~\ref{appendix:gpt4-eval}.

As shown in \cref{tab:hh-main}, except DPO and PPO, we also investigate other alignment methods such as RRHF~\cite{rrhf} and PRO~\cite{pro}. The results demonstrate that PPO and DPO are much more preferred by GPT-4 than the chosen responses in the dataset and SFT model outputs, outperforming RRHF and PRO across all metrics. In this paper, we focus more on the performance of DPO and PPO. We observe that DPO-Iter performs better than DPO but worse than PPO. PPO consistently achieves a higher reward and higher win rates. We also use GPT-4 to compare the outputs of DPO and PPO directly, and the results are listed in \cref{tab:hh-vs}, which demonstrates that GPT-4 prefers the responses of PPO.

\textbf{SafeRLHF}~\cite{dai2023safe} dataset comprises over 30k entries of expert comparison data. Each entry in this dataset contains two responses to a question.
In our experiments, we consolidate two preferences as mentioned in Section~\ref{subsec:analysis-dpo-exp}.
For evaluation, we borrow the official reward model and cost model\footnote{https://github.com/PKU-Alignment/safe-rlhf}, which are trained to evaluate helpfulness and harmlessness, respectively. 

The results on SafeRLHF are listed in Tab~\ref{tab:safe-main}. 
Experiments indicate that after alignment, both DPO and PPO can generate responses with less harm, while PPO's responses are more helpful.

\textbf{APPS}~\cite{hendrycks2021measuring} is a description-to-code
generation benchmark from competitive programming platforms. For each question, there are also test cases to verify the accuracy of generated codes.
We use these test cases in the training set to provide feedback.
For PPO training, the feedback could be directly used as a reward. We simply define the reward as 10 if the generated code passes \textbf{all} test cases. Otherwise, the reward is 0. 
For DPO, since there are no preference pairs, we adopt DPO-Iter.
Specifically, we use the base model to sample 5 codes for each prompt and utilize the test cases to label the correctness of generated codes. 
It is worth noting that for many prompts, the base model may fail to sample any correct answer. In such cases, we use the correct solutions from the dataset as $\mathbf{y}_w$. We evaluate the results using pass@k, which is defined as the proportion of problems successfully solved by employing k generated programs for each problem.

As shown in \cref{tab:main_apps}. We conduct experiments on different model sizes. In particular, when using CodeLlama-34B as the base model, we achieved state-of-the-art results on the APPS dataset. 
We can observe that DPO-Iter fails to improve the SFT model performances on all the model sizes. In contrast, for PPO, as the model size increases, the improvement is more apparent. We remark that the feedback using test cases is nearly perfect. However, the performance of DPO-Iter remains unsatisfactory.

\textbf{CodeContest}~\cite{li2022competition} is a more challenging competitive programming dataset consisting of several programming languages. Here, we only use Python code. We adopt a similar way to train PPO as in the APPS dataset. For DPO training, we construct the preference dataset by using the correct and incorrect codes provided by the dataset. 
To compare with previous work, we adopt k@n to evaluate the generated code, which means that n samples will be evaluated on public tests in the problem description, and k of them will be submitted for hidden tests.

The results are listed in \cref{tab:ccs-baseline}. We obtained similar conclusions as in APPS. PPO improves the SFT model significantly, while DPO fails to generate any correct codes. After one epoch of training, the code written by the DPO model has achieved a pass rate of 0, we observe that the DPO model outputs many meaningless code snippets.
The results also demonstrate that DPO-Iter performs worse compared to SFT.
With the assistance of PPO, CodeLlama-34B has surpassed the previous state-of-the-art on this task, outperforming Alphacode with 41 billion parameters.

%% file: tables/safe_main.tex
\begin{table}[ht]
\centering
\small
\begin{tabular}{c|c|ccc}
\toprule
LLM & Method & $\Delta$Help. $\uparrow$ & Harm. $\downarrow$  & S.R. $\uparrow$ \\
\midrule
& Beaver & -  & -6.59  & 89.6\%      \\
\midrule
\multirow{4}{*}{\makecell{Llama 1 \\ 7B}}
& SFT  & -2.26 & 0.78 & 46.5\% \\
& DPO & -2.70 & -6.38 & 93.1 \% \\
& DPO-Iter  & -2.79 & \textbf{-11.86 } & \textbf{100.0\%} \\
& PPO & \textbf{+0.66}  & -10.22 & 98.6\%      \\
\midrule
\multirow{4}{*}{\makecell{Llama 2 \\ 7B}}
& SFT & -2.12 & 0.00 & 52.1\% \\
& DPO  & -2.86 & -6.82  & 95.8\% \\
& DPO-Iter    & -2.96 & -11.07  & \textbf{99.9\%} \\
& PPO  & \textbf{+1.69}  & \textbf{-12.08} & 99.5\%      \\
\bottomrule
\end{tabular}
\vspace{-2mm}
\caption{Results on SafeRLHF. ``Beaver'' is the officially released model. ``$\Delta$ Help.'' denotes helpfulness relative to Beaver. ``S.R.'' denotes safety rate. The reported results are based on the official evaluation model.}
\vspace{-3mm}
\label{tab:safe-main}
\end{table}

%% file: tables/apps_main.tex
\begin{table}[ht]
\centering
\small
\begin{tabular}{cc|ccc}
\toprule
Model & Method & Intro. & Inter. & Comp. \\
\midrule
\multirow{1}{*}{\makecell{GPT-Neo 2.7B}}
& SFT & 5.6\% & 0.8\% & 0.0\% \\
Codex 12B & SFT & 9.7\% & 0.5\% & 0.1\% \\
CodeT5 & CodeRL & 16.4\% & 4.9\% & 2.8\% \\
\makecell{AlphaCode 1B} &  5@1k & 14.4\% & 5.6\% & 4.6\% \\
\midrule
\multirow{4}{*}{\makecell{Code Llama\\7B}}
& \multirow{1}{*}{Few shot}
& 10.8\% & 2.0\% & 0.8\% \\
& \multirow{1}{*}{SFT}
& \textbf{30.0\%} & \textbf{7.8\% }& \textbf{2.8\%} \\
& DPO-Iter & 20.9\% & 3.4\% & 1.3\%\\
& \multirow{1}{*}{PPO}
&29.4\% & 7.6\% & 2.4\% \\
\midrule

\multirow{4}{*}{\makecell{Code Llama\\13B}}
& \multirow{1}{*}{Few shot}
& 23.7\% & 5.6\% & 2.1\% \\
& \multirow{1}{*}{SFT}
& 33.7\% & 8.7\% & 3.6\% \\
& DPO-Iter & 33.0\% & 8.0\% & 2.8\%\\
& \multirow{1}{*}{PPO}
&\textbf{36.4\%} & \textbf{11.47\%} & \textbf{4.6\%} \\
\midrule

\multirow{4}{*}{\makecell{Code Llama\\34B}}
& \multirow{1}{*}{Few shot}
& 32.8\% & 8.8\% & 2.9\% \\
& \multirow{1}{*}{SFT}
& 38.6\% &  10.1\% & 3.9\% \\
& DPO-Iter & 34.2\% & 9.3\% & 3.7\% \\
& \multirow{1}{*}{PPO}
&\textbf{44.4\%} & \textbf{18.0\%} & \textbf{9.1\%} \\
\midrule

\end{tabular}
\vspace{-2mm}
\caption{Results on Apps test set. All the numbers are pass@5 except for AlphaCode. Where ``5@1k'' means this model samples 1000 times for each problem and 5 sampled codes that pass the public test cases (in the problem description) are selected to be evaluated on hidden test cases.}
\vspace{-3mm}
\label{tab:main_apps}
\end{table}

%% file: tables/ccs_main.tex
\begin{table}[ht]
\centering
\small
\begin{tabular}{cc|cc}
\toprule
     Model & Method & \thead{Valid. Set\\10@1k} & \thead{Test Set\\10@1k}\\
    \midrule
     AlphaCode 9B & - & 16.9\% & 14.3\%\\
     \midrule
     \multirow{2}{*}{\makecell{AlphaCode 41B}}& - & 16.9\% & 15.6\%\\
     & + clustering & 21.0\% & 16.4\%\\
     \midrule
     \multirow{4}{*}{Code Llama 34B} & SFT & 10.3\% & 15.2\% \\
     & DPO & 0.0\% & 0.0\% \\
     & DPO-Iter & 3.5\% & 3.2\% \\
     & PPO & 19.7\% & \textbf{22.4}\% \\
     \bottomrule
\end{tabular}
\caption{Pass rate on CodeContests dataset. ``10@1k'' means that 1000 samples will be evaluated on public tests in the problem description, and only 10 of them will be submitted for hidden tests. We \textbf{only used Python for solving problems}, while AlphaCode used both Python and C++.}
\vspace{-5mm}
\label{tab:ccs-baseline}
\end{table}

%% file: 60_conclusion.tex
\section{Conclusion}
\label{conclusion}
In this paper, we uncover the fundamental limitations of DPO and explore critical factors that enhance the practical performance of PPO in RLHF. Through theoretical and experimental analysis, we explore the limitations of DPO and find that DPO is sensitive to the distribution shift between the base model outputs and preference data. We suggest that iterative DPO is better than training on static data. However, we also find that DPO fails to improve the performance on challenging tasks such as code generation. 
Moreover, according to the ablation study, we summarize the key factors for PPO training, including advantage normalization, large batch size, and updating the parameters of the reference model with an exponential moving average. With our practical tuning guideline, PPO demonstrates robust effectiveness across diverse tasks and achieves state-of-the-art results in challenging code competition tasks.

There are also limitations in our work. The reward model is significant in the training processes of both PPO and DPO-Iter. However, in this paper, we have not delved into the discussion of how to effectively train a robust reward model. For the code competition task, we utilize the ground-truth reward for PPO training and the labeling of DPO-Iter. However, this does not affect the conclusions drawn in our paper, and we leave it as future works.

%% file: 70_appendix.tex
\newpage
\appendix
\onecolumn
\section{Implementation Details}

\subsection{DPO Details}
For DPO training, we use $\beta$ = 0.1  with a learning rate of 1e-6. We sweep the batch size and report the best performance. For HH-RLHF and SafeRLHF, we train DPO for two epochs. For code generation tasks, we train DPO for a single epoch, since it has led to a deterioration in performance.

\subsection{PPO Details}
\label{appendix:ppo_details}
During the PPO training phase, we separate the parameters of actor and critic, and set the learning rate to 1e-5 for the actor model and 5e-6 for the critic model. By default, we set the global batch size as 512, and 512 roll-out samples are split into 4 mini-batches to update the actor and critic models.
We configure the sampling parameters to include a temperature of 1.0 and a top-k value of 200. 
The advantage estimation parameter $\lambda$ in GAE and the RL discount factor $\gamma$ are fixed at 1. We set the KL penalty coefficient $\beta$ as 0.1, with a clipping value of 20 for reward scores. We additionally adopt advantage normalization and value normalization to stabilize the training.

For HH-RLHF and SafeRLHF, we set the maximum generated tokens as 256 and adopted PPO training for 5 epochs. 
For APPS and CodeContest, we set the maximum generated tokens as 1024, and adopt PPO training for 16 epochs.
The checkpoints with the highest reward/pass@k on the validation sets are selected.

\section{GPT-4 Evaluation}
\label{appendix:gpt4-eval}
We adopt the same evaluation prompt with \cite{dpo}. The prompt is :
\begin{verbatim}
For the following query to a chatbot, which response is more helpful?

Query: <the user query>

Response A:
<either the test method or baseline>

Response B:
<the other response>

FIRST provide a one-sentence comparison of the two responses and explain \
which you feel is more helpful. SECOND, on a new line, state only "A" or \
"B" to indicate which response is more helpful. Your response should use \
the format:
Comparison: <one-sentence comparison and explanation>
More helpful: <"A" or "B">
\end{verbatim}

When using GPT-4 for evaluation, we randomly sampled 100 queries from the test set. And ask GPT-4 to compare the two responses. To minimize the impact of response position on comparison, we swapped the positions of the two responses and evaluated them separately. If the results of the two evaluations are inconsistent, we set the final result as a ``Tie''.

\section{Additional Experiments}

\input{tables/apps_reference}
\input{tables/safe_reference}
\subsection{Varying the Reference Model}
We conduct experiments to assess the impact of distribution shift by varying the reference model. The results are listed in Table~\ref{tab:apps-reference} and Table~\ref{tab:safe-reference}. Llama2-7B-SFT(Safe) and Codellama13B-SFT are models that are closer to the preference dataset in the Safe-RLHF and APPS dataset, respectively. The results indicate that DPO is more affected by the distribution shift than PPO.

\subsection{Varying $\beta$}
In Table~\ref{tab:ablation-beta}, We explore the impact of $\beta$ on the HH-RLHF and APPS datasets. On the HH-RLHF dataset, we evaluate the model using the OpenAssistant reward metric. On the APPS dataset, we report the average pass@5 score.
The results indicate that having too large $\beta$ may harm the performance of both DPO and PPO. A $\beta$ value of 0.1 consistently performs well across various models and tasks.
\input{tables/beta}

\subsection{Varying Preference Dataset}
\input{tables/coverage}
We train the model on a subset of the HH-RLHF preference dataset. The results are shown in Table~\ref{tab:ablation-coverage}.    The results suggest that the performance of both PPO and DPO may be affected by the extent of coverage in the preference dataset. When training on the helpful-base subset, the performance of DPO has dropped to be similar to that of the SFT model.

We also evaluate PPO on the HH-RLHF dataset by filtering dual-unsafe and dual-safe preference pairs. The results are listed in Table~\ref{tab:filter}. We observe that PPO could also be affected by the composition of the preference dataset. Overall, PPO maintains a safe rate of over 92\% cross all the settings, while DPO is more affected by the preference dataset.

When filtering dual-unsafe samples, the PPO model achieves significantly higher helpfulness rewards. We hypothesize that it is because the reward model can discern helpfulness at a more nuanced level. Upon further filtering of dual-safe samples, we observe that the model becomes conservative, often declining to respond to questions altogether. This phenomenon occurs because, after filtering both dual-unsafe and dual-safe samples, the reward model focuses solely on safety. And refusing to respond could always be a safe option.

\input{tables/filter}

\subsection{Human Evaluation}
We also include human evaluation to validate the preference-based tasks. The results are listed in Table~\ref{tab:human-eval}. We ensure that each reference pairs are evaluated by 4 different persons.
Human agree with GPT-4 evaluations at a rate of 60\% and 61\%, respectively. According to human evaluation results, PPO outperforms both DPO and DPO-Iter.
\input{tables/hh_human_eval}

%% file: tables/apps_reference.tex
\begin{table}[ht]
\centering
\small
\begin{tabular}{c|c|c}
\toprule
Method & Reference Model & Avg. Pass@5\\
\midrule
DPO& Codellama-13B-Pretrain & 0.24\% \\
DPO& Codellama-13B-SFT & 12.8\% \\
\midrule
PPO& Codellama-13B-Pretrain & 13.8\% \\
PPO& Codellama-13B-SFT & 15.1\% \\
\bottomrule
\end{tabular}
\caption{Results of changing the reference model on APPS dataset. Codellama-13B-SFT is closer to the preference dataset than Codellama-13B-Pretrain. DPO is more affected by the distribution shift than PPO.}
\label{tab:apps-reference}
\end{table}

%% file: tables/safe_reference.tex
\begin{table}[ht]
\centering
\small
\begin{tabular}{cc|ccc}
\toprule
Method & Reference Model & $\Delta$Help. $\uparrow$ & Harm. $\downarrow$  & S.R. $\uparrow$ \\
\midrule
DPO & Llama2-7B-SFT(Alpaca) & -4.19 & -0.97 & 55.4\% \\
DPO & Llama2-7B-SFT (Safe) & -1.62 & -3.5 & 71.8\% \\
\midrule
PPO & Llama2-7B-SFT(Alpaca) & 1.69 & -12.08 & 99.5\% \\
PPO & Llama2-7B-SFT (Safe) & 4.47 & -12.33 & 99.6\% \\
\bottomrule
\end{tabular}
\caption{Results of changing the reference model on Safe-RLHF dataset. Llama2-7B-SFT(Safe) is closer to the preference dataset than Llama2-7B-SFT(Alpaca). DPO is more affected by the distribution shift than PPO.}
\label{tab:safe-reference}
\end{table}

%% file: tables/beta.tex
\begin{table}[ht]
\centering
\small
\begin{tabular}{c|cccc}
\toprule
$\beta$ & 0 & 0.05 & 0.1 (default) & 0.2 \\
\midrule
HH-RLHF, Llama-7B & & & &  \\
\midrule
PPO & 0.705 & 0.720 & 0.718 & 0.629 \\
DPO & N/A & 0.609 & 0.611 & 0.597 \\
\midrule
APPS, Codellama-13B & & & & \\
\midrule
PPO & 13.0\% & 14.1\% & 15.1\%  & 14.9\% \\
DPO & N/A & 12.56\% & 12.0\% & 12.32\% \\
\bottomrule
\end{tabular}
\caption{Results of changing the $\beta$ parameter on HH-RLHF and APPS dataset. The results indicate that having too large $\beta$ may harm the performance of both DPO and PPO. A $\beta$ value of 0.1 consistently performs well across various models and tasks.}
\label{tab:ablation-beta}
\end{table}

%% file: tables/coverage.tex
\begin{table}[ht]
\centering
\small
\begin{tabular}{c|cc}
\toprule
Preference Dataset & helpful-base set & full set (default) \\
\midrule
SFT & N/A & 0.532 \\
PPO & 0.602 & 0.718 \\
DPO & 0.544 & 0.615 \\
\bottomrule
\end{tabular}
\caption{Results of changing the coverage level of preference dataset on HH-RLHF dataset. When training on the helpful-base subset, the performance of DPO has dropped to be similar to that of the SFT model.}
\label{tab:ablation-coverage}
\end{table}

%% file: tables/filter.tex
\renewcommand{\sftalpaca}{{\emph{SFT (Alpaca)}}\xspace}
\renewcommand{\sftsafe}{{\emph{SFT (Safe)}}\xspace}

\begin{table}[ht]
\centering
\small
\begin{tabular}{c|ccc}
\toprule
   & $\Delta$Help. $\uparrow$ & Harm. $\downarrow$  & S.R. $\uparrow$ \\ \midrule
PPO & 1.69 & -12.08 & 99.5\% \\
+ filter dual-unsafe & 5.88 & -9.12 & 92.6\% \\
+ filter dual-safe & -8.04 & -4.51 & 94.9\% \\
\midrule
DPO & -1.62 & -3.50  & 71.8\%      \\
+ filter dual-unsafe & 2.46  & -4.88  & 80.8\%      \\ 
+ filter dual-safe & -2.86  & -6.82 & 95.8\%      \\ 
\bottomrule
\end{tabular}
\caption{The impact of filtering dual-safe and dual-unsafe training data on PPO and DPO on the Safe-RLHF dataset.}
\label{tab:filter}
\end{table}

%% file: tables/hh_human_eval.tex
\begin{table}[ht]
\centering
\small
\begin{tabular}{ccccc}
\toprule
 & PPO win & Tie & DPO win & GPT4-Human agree \%\\
\midrule
PPO V.S. DPO & 45 & 26 & 29 & 60 \\
PPO V.S. DPO-iter & 38 & 29 & 33 & 61 \\
\bottomrule
\end{tabular}
\caption{Results of human evaluation on HH-RLHF dataset. GPT-4  agrees with human evaluations at a rate of 60\% and 61\%, respectively. According to human evaluation results, PPO outperforms both DPO and DPO-Iter.}
\label{tab:human-eval}
\end{table}